\documentclass[runningheads]{llncs}
\usepackage[T1]{fontenc}
\usepackage{graphicx}
\usepackage{hyperref}
\usepackage{color}

\usepackage[square,sort,comma,numbers]{natbib}
\usepackage{caption}
\usepackage{booktabs}

\usepackage{amsmath}

\usepackage{algorithm}
\usepackage[noend]{algpseudocode}
\MakeRobust{\Call}

\algrenewcommand\algorithmicrequire{\textbf{Input:}}
\algrenewcommand\algorithmicensure{\textbf{Output:}}

\usepackage{subcaption}

\usepackage[printonlyused]{acronym}
\acrodef{BNC}{Bayesian network classifier}
\acrodef{OBDD}{ordered binary decision diagram}
\acrodef{NNF}{negation normal form}
\acrodef{DAG}{directed acyclic graph}
\acrodef{SR}{sufficient reason}
\acrodef{NR}{necessary reason}
\acrodef{GSR}{general sufficient reason}
\acrodef{GNR}{general necessary reason}

\usepackage[noabbrev, capitalise]{cleveref}
\crefname{ALC@unique}{line}{lines}
\Crefname{ALC@unique}{Line}{Lines}

\usepackage[pdf]{graphviz}

\usepackage{xpatch}
\makeatletter
\newcommand*{\addFileDependency}[1]{
  \typeout{(#1)}
  \@addtofilelist{#1}
  \IfFileExists{#1}{}{\typeout{No file #1.}}
}
\makeatother
\xpretocmd{\digraph}{\addFileDependency{#2.dot}}{}{}

\usepackage{xcolor}

\newcommand\nil{\texttt{nil}}
\newcommand\true{\texttt{true}}
\newcommand\false{\texttt{false}}
\newcommand\dom{\texttt{dom}}
\newcommand{\ftw}[1]{\texttt{ftw}(#1)}

\newcommand\vNormal{\texttt{Normal}}
\newcommand\vHigh{\texttt{High}}
\newcommand\vLow{\texttt{Low}}

\newcommand\BP{\textsc{BP}}
\newcommand\HR{\textsc{HR}}
\newcommand\CT{\textsc{CT}}

\newcommand\I{\mathcal{I}}
\newcommand{\X}{\mathbf{X}}
\newcommand{\x}{\mathbf{x}}
\newcommand{\Sep}{\mathbf{S}}
\newcommand{\sep}{\mathbf{s}}
\newcommand{\U}{\mathbf{U}}
\newcommand{\V}{\mathbf{V}}
\newcommand{\Z}{\mathbf{Z}}
\newcommand{\pr}{\texttt{pr}}
\newcommand{\uu}{\mathbf{u}}
\newcommand{\vv}{\mathbf{v}}

\begin{document}
\title{Scaling the Explanation of Multi-Class Bayesian Network Classifiers}
%
%
\author{Yaofang Zhang\orcidID{0009-0001-2686-8954}
\and
Adnan Darwiche\orcidID{0000-0003-3976-6735}
}
\authorrunning{Y. Zhang, A. Darwiche}
%
\institute{
Department of Computer Science\\
University of California, Los Angeles\\
\email{\{vzhang,darwiche\}@cs.ucla.edu}}
\maketitle              
\begin{abstract}
We propose a new algorithm for compiling \acf{BNC} into \textit{class formulas}.
Class formulas are logical formulas that represent a classifier's input-output behavior, and are crucial in the recent line of work that uses logical reasoning to explain the decisions made by classifiers. 
Compared to prior work on compiling class formulas of \ac{BNC}s, our proposed algorithm is not restricted to binary classifiers, shows significant improvement in compilation time, and outputs class formulas as \acf{NNF} circuits that are \textit{OR-decomposable}, which is an important property when computing explanations of classifiers.
\end{abstract}

\section{Introduction}

As machine learning systems become more prevalent in everyday interactions, so has the interest of reasoning about the behavior of such systems. 
The decision of a (trained) classifier may be used to help doctors diagnose a medical condition; in such scenarios, it is critical to explain the reason behind the classifier's decision, e.g., what is the minimal change to the classifier input that will change the decision, and what are the minimal aspects of the input that guarantee this decision?
Recent works in explainable AI have proposed to answer such questions with hard guarantees by first constructing a symbolic model for the input-output behavior of a classifier, e.g.,~\cite{shih_symbolic_2018,Ignatiev_abductive_2019,Audemard_xai_compiled_2020}.\footnote{This is in contrast to the model-agnostic approach which treats a classifier as a black box that only can be queried~\cite{Ribeiro_lime_2016, Ribeiro_Anchors_2018}. The model-agnostic approach can scale better, but the generated explanations are approximate and lack hard guarantees~\cite{corr/Ignatiev_validate_2019}.}
In particular, this paper follows a line of work that proposed to efficiently generate explanation for classifier decisions by compiling the \textit{class formulas}, which represent the input-output behavior of classifiers, into \textit{tractable circuits}~\cite{shih_symbolic_2018, Audemard_xai_compiled_2020, darwiche_computation_2022}.
This paper will focus on compiling class formulas for one influential type of classifiers, the \acf{BNC}, which has been widely used in both the AI field and beyond \cite{friedman_bayesian_1997, ng_discriminative_2001, sachs_causal_2005, bielza_discrete_2014}.

Prior works have proposed algorithms for compiling the class formulas of naive Bayes \cite{chan_reason_bayes_2003,nips/marques2020naivebayes}, latent tree \cite{shih_symbolic_2018}, and general \ac{BNC} \cite{shih_compiling_2019}.
In prior works, the \acp{BNC} are assumed to be binary (i.e., two possible output classes), and the class formulas are compiled to \acp{OBDD}.
In this paper, we present a new algorithm for compiling general \acp{BNC} that: 
(1) generalizes from binary \acp{BNC} to multi-class \acp{BNC}, 
(2) shows significant improvement in compilation time, and
(3) outputs class formulas that are \textit{OR-decomposable,} which is an important property for the tractability of computing explanations~\cite{jair/DarwicheM21,darwiche_computation_2022}. 
In particular, if the instances in each class are represented by an OR-decomposable formula, complete reasons~\cite{ecai/DarwicheH20,darwiche_complete_2023} and general reasons~\cite{ji_new_2023} of decisions can be computed in linear time. 
These can then be converted into sufficient and necessary reasons which are perhaps the most common explanations for classifier decisions that are studied in the literature; see~\cite{lics/Darwiche23} for a tutorial/survey.

The rest of the paper is organized as follows. 
\cref{sec:bg} introduces the problem of compiling class formulas for \ac{BNC}s.
\cref{sec:ftree} introduces the feature tree, which is the central component for the compilation algorithm discussed in \cref{sec:algMain}.
Experimental results are presented in \cref{sec:results}.
\Cref{sec:multiclass} presents an example multi-class BNC, and illustrates explanations that can be efficiently generated using the output of the compilation algorithm.
Finally, \cref{sec:conclusion} concludes our discussion.

\section{Background}
\label{sec:bg}

In this section, we will first introduce the \acf{BNC} and the jointree method for probabilistic inference in Bayesian networks.
We then describe how the compiled class formulas can be used to explain the decisions of classifiers.
We will end with a summary of the \ac{BNC} compilation algorithm proposed in prior work.

\subsection{Bayesian Network Classifier}

A \ac{BNC} is a tuple $(\mathcal{N}, \X, Y)$, where $\mathcal{N}$ is a Bayesian network, $\X$ is the set of input feature variables, and $Y$ is the output variable that is the classification target.
In general, the network variables are discrete.
A discrete variable $Z$ has $k$ possible states/values $\{z_0, \dots, z_{k-1}\}$.
When variable $Z$ is binary, the two states are also commonly denoted as $\lnot Z$ and $Z$.
For some state $y_i$ of the target $Y$, an instance $\x$ is in class $y_i$ iff $i = \arg\max_{j} p(y_j|\x)$.
For class $y_i$, instance $\x$ is positive if $\x$ belongs to $y_i$; otherwise, $\x$ is negative. 

Following previous works, we assume that the classification target $Y$ is one of the root nodes in the network, and the input features are among the leaf nodes.
For our discussion, unless stated otherwise, all leaf variables of the network are chosen as features.\footnote{If a subset of leaves is chosen, we can iteratively remove other leaf nodes without impacting the classifier~\cite[Ch.~6]{darwiche_modeling_2009}.}
We will use the network in \cref{fig:examplebn} as a running example with $Y = A$ and $\X = \{D, E, F\}$.

\begin{figure}
\centering
\digraph[scale=0.3]{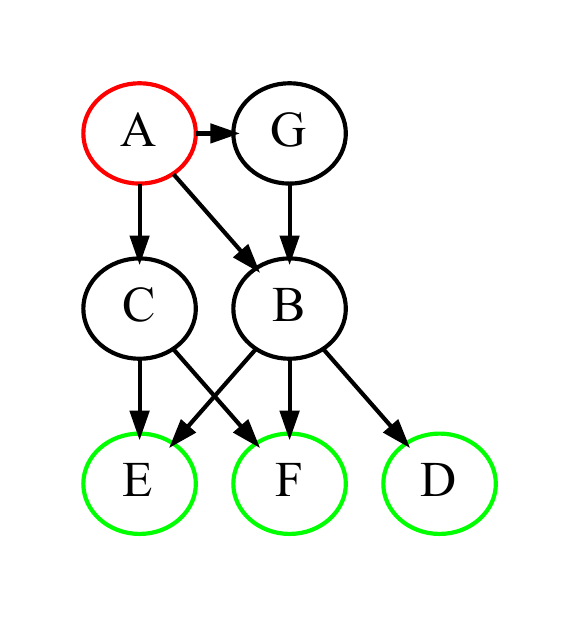}{
A -> G[penwidth=2]; A -> B[penwidth=2]; A -> C[penwidth=2];
B -> D[penwidth=2]; B -> E[penwidth=2]; B -> F[penwidth=2]; 
G -> B[penwidth=2]; C -> E[penwidth=2]; C -> F[penwidth=2];
{rank=same; A [color=red fontsize=24 penwidth=2]; G [fontsize=24 penwidth=2];}
{rank=same; B [fontsize=24 penwidth=2]; C [fontsize=24 penwidth=2];}
{rank=same; D [color=green fontsize=24 penwidth=2]; E [color=green fontsize=24 penwidth=2]; F [color=green fontsize=24 penwidth=2];}}
\caption{An example Bayesian network. The root variable A is in red, and the leaf variables D, E, F are in green. For simplicity, all variables in this network are binary.}
\label{fig:examplebn}
\end{figure}

The jointree algorithm is an influential method of probabilistic inference in Bayesian networks \cite{Lauritzen_jtree_1988, darwiche_modeling_2009}, and the jointree is also the basis for the feature tree that will be introduced in \cref{sec:ftree}.
\Cref{fig:ftree} shows a jointree for the network in \cref{fig:examplebn}.
Nodes in a jointree represent \textit{clusters}, which are sets of network variables, and each family (a variable and its parents) in the network must be a subset of some cluster in the jointree.
Also, if a variable appears in two clusters in a jointree, every cluster on the path between the two clusters must include the variable.
Each edge in a jointree is associated with a \textit{separator}, which is the intersection of the clusters of the two nodes that are connected by the edge.
Let $\Sep$ be the separator variables of an edge, variables $\Z$ be the union of the clusters that appear on one side of the edge, and $\Z'$ be the union of the clusters that appear on the other side.
Separator $\Sep$ \textit{d-separates} $\Z \setminus \Sep$ and $\Z' \setminus \Sep$; that is, given separator $\Sep$, the variables that only appear on one side of the edge are independent from the variables that only appear on the other side; see, e.g., \cite[Ch.~7]{darwiche_modeling_2009}.
The jointree algorithm manipulates factors over clusters and factors over separators.
A factor over a set of variables $\X$ is a mapping from each instantiation $\x$ to a non-negative number.
Factors can be used to represent probability distributions. 
Multiplying two factors creates a factor over the union of the variables in the two multiplicands. 
Projecting a factor to a variable set sums out all the variables that are not in the set.
The jointree algorithm can compute the marginals over all clusters, and its time complexity is exponential with respect to the jointree width (i.e., size of the largest cluster minus one). 

\subsection{Explaining Decisions using Class Formulas}

Even though \ac{BNC}s use probabilistic reasoning to classify feature instances, the input features $\X$ and the output target $Y$ are (symbolic) discrete variables. 
For a class $y_i$, the set of positive instances can be represented by a logical formula $\Pi_i$, which is called the \textit{class formula} of $y_i$.
The set of positive instances for class $y_i$ is exactly the set of models of 
$\Pi_i;$ that is, the instantiations $\x$ that satisfy formula $\Pi_i.$

The class formula is a key ingredient to explaining classifier decisions.
Continuing our example with the network in \cref{fig:examplebn}, with $Y = A$ and $\X = \{D, E, F\}$. 
Suppose an instance $\x = (D, \lnot E, \lnot F)$ is classified to $y_1$, and the class formula $\Pi_1$ is as shown in \cref{fig:nnf2}, so $\Pi_1=(D \lor \lnot E) \land (\lnot D \lor \lnot E \lor F)$.
The \textit{complete reason} for this decision is $(D \lor \lnot E) \land (\lnot E \lor F)$, 
which can be viewed as an abstraction of $\x$ showing why $\x$ belongs to the decided class~\cite{ecai/DarwicheH20,darwiche_complete_2023}.
The \acfp{NR} are the prime implicates of the complete reason, which are two in this case, $D \lor \lnot E$ and $\lnot E \lor F$.\footnote{The formulation of necessary reasons as the prime implicates of the complete reason is due to~\cite{darwiche_computation_2022}, which also showed that they correspond to contrastive explanations when the classifier is binary. Contrastive explanations were initially discussed in~\cite{Lipton_contrastive_1990} then formalized and popularized in~\cite{Ignatiev_Contrastive_2020}.}
A necessary reason points to a minimal modification to the instance, such that if a necessary reason is violated, the modified instance is not guaranteed to belong to the same class.
For example, if the values for variables $D$ and $E$ are flipped in $\x$, which violates a necessary reason ($D \lor \lnot E$), the modified instance $\x' = \{\lnot D, E, F\}$ is not in the class $y_1$.
This modification is minimal in the sense that it is impossible to change the decided class by changing the value of only variable $D$ or only variable $E$.
Furthermore, the \acfp{SR}, which are the prime implicants of the complete reason, correspond to the minimal aspects of the instance that guarantee/justify the decision.\footnote{The formulation of sufficient reasons as the prime implicants of the complete reason is due to~\cite{ecai/DarwicheH20}. Sufficient reasons were originally introduced under the name of PI-explanations in~\cite{shih_symbolic_2018} and were later referred to as abductive explanations in~\cite{Ignatiev_abductive_2019}.}
The sufficient reasons in this example are $\lnot E$ and $D \land F$. 
As long as one of the sufficient reasons holds, the decided class does not change no matter how the instance is modified.
These notions are further generalized to general reasons, \acfp{GSR}, and \acfp{GNR}~\cite{ji_new_2023}, which are more informative when the variables are non-binary.
In \cref{sec:multiclass}, we will discuss them in more depth in the context of a multi-class \ac{BNC} with discrete features.
Besides \acp{BNC}, this approach of explaining decisions using class formulas has also been applied to classifiers such as decision graphs~\cite{darwiche_computation_2022,ji_new_2023}, random forests~\cite{XAI26/jiRF}, and binary neural networks~\cite{KR20/shiCompileBNN}.

Efficiently computing the complete/general reason requires the class formula to satisfy certain properties.
We represent logic formulas as \ac{NNF} (circuits), which contain only conjunctions, disjunctions and negations with the latter appearing only next to variables (see \cref{fig:classFormulaNNF}).
An NNF is AND-decomposable iff for every conjunction $\land_{i} \alpha_i$ in the NNF, the conjuncts $\alpha_i$ do not share variables \cite{jacm/Darwiche01}.
\cref{fig:nnf1} shows an AND-decomposable NNF.
Similarly, an NNF is OR-decomposable iff for every disjunction $\lor_{i} \alpha_i$, the disjuncts $\alpha_i$ do not share variable; see
\cref{fig:nnf2}.
Recent works show that OR-decomposability allows one to efficiently compute the complete and general reasons for a decision \cite{jair/DarwicheM21,darwiche_computation_2022, ji_new_2023}, which are used to compute the (general) sufficient and necessary reasons for decisions.

\begin{figure}
\centering
\begin{subfigure}{0.48\textwidth}
\centering
\digraph[scale=0.25]{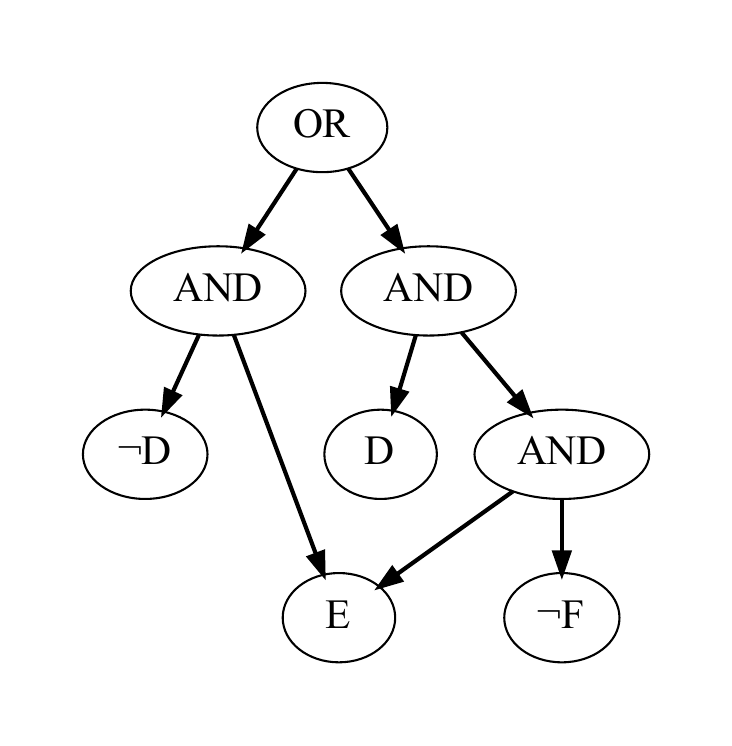}{
	11 [label=OR fontsize=20]
	11 -> 8 [style=solid penwidth=2.0]
	8 [label=AND fontsize=20]
	8 -> 2 [style=solid penwidth=2.0]
	2 [label="¬D" fontsize=20]
	8 -> 5 [style=solid penwidth=2.0]
	5 [label=E fontsize=20]
	11 -> 10 [style=solid penwidth=2.0]
	10 [label=AND fontsize=20]
	10 -> 3 [style=solid penwidth=2.0]
	3 [label=D fontsize=20]
	10 -> 9 [style=solid penwidth=2.0]
	9 [label=AND fontsize=20]
	9 -> 5 [style=solid penwidth=2.0]
	9 -> 6 [style=solid penwidth=2.0]
	6 [label="¬F" fontsize=20]
}
\caption{Negative instances of $y_1$.}
\label{fig:nnf1}   
\end{subfigure}%
\begin{subfigure}{0.48\textwidth}
\centering
\digraph[scale=0.25]{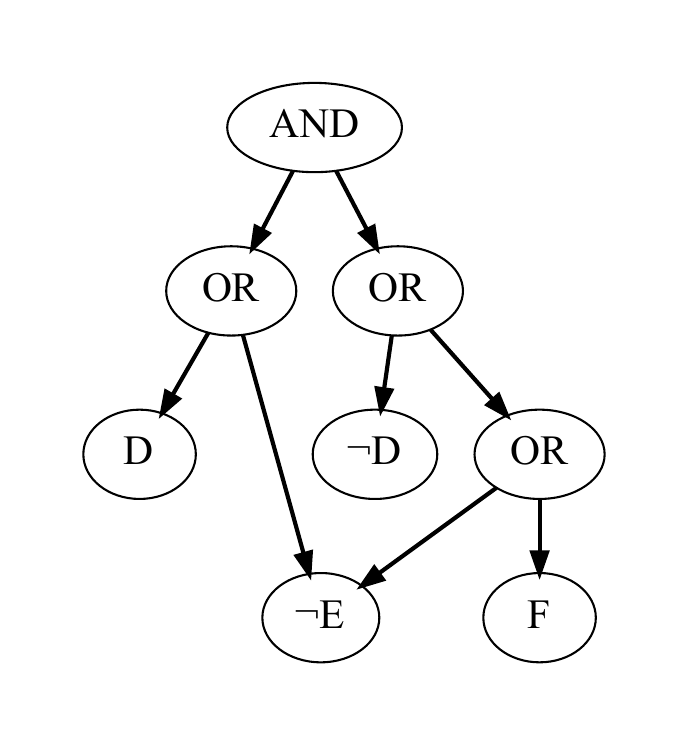}{
	12 [label=AND fontsize=20]
	12 -> 8 [style=solid penwidth=2.0]
	8 [label=OR fontsize=20]
	8 -> 2 [style=solid penwidth=2.0]
	2 [label=D fontsize=20]
	8 -> 4 [style=solid penwidth=2.0]
	4 [label="¬E" fontsize=20]
	12 -> 11 [style=solid penwidth=2.0]
	11 [label=OR fontsize=20]
	11 -> 3 [style=solid penwidth=2.0]
	3 [label="¬D" fontsize=20]
	11 -> 10 [style=solid penwidth=2.0]
	10 [label=OR fontsize=20]
	10 -> 4 [style=solid penwidth=2.0]
	10 -> 7 [style=solid penwidth=2.0]
	7 [label="F" fontsize=20]
}
\caption{Class formula of $y_1$.}
\label{fig:nnf2}   
\end{subfigure}%
\caption{(a) An AND-decomposable NNF for the negative instances of $y_1$, and (b) its negation, which is an OR-decomposable NNF of the class formula for $y_1$.}
\label{fig:classFormulaNNF}
\end{figure}

\subsection{Prior Work: Block-Order Based Compilation}

An algorithm for compiling general \ac{BNC}s into class formulas was proposed in~\cite{shih_compiling_2019}, which is the state of the art before our reported work.
The core concept underlying the algorithm is that of ``equivalent sub-classifiers.''
For a classifier $F(\X)$ and an instance $\x=(\uu,\vv)$, by fixing a part of the instance ($\uu$), we get a sub-classifier $F_\uu(\V)$ over a smaller set of features. 
Caching (and re-using already compiled circuits) depends on testing equivalence of some sub-classifiers $F_\uu$ and $F_{\uu'}$.
The algorithm explores the features $\X$ in ``block ordering'' $(\X_1, ..., \X_B)$.
Block ordering means that for each $k < B$, there is a variable $H^k \not \in \X$, that ``splits'' the features $\X$ into $(\X_1 \cup ... \cup \X_k)$ and $(\X_{k+1} \cup ... \cup \X_B)$.
A variable $H$ splits features $\X$ into $(\U, \V)$ iff $H$ d-separates features $\V$ from $\U \cup \{Y\}$.
The algorithm compiles sub-classifiers $F_\U(\V)$ for $\U = \X_1 \cup ... \cup \X_k$ and $\V = \X_{k+1} \cup ... \cup \X_B$.
For the network in \cref{fig:examplebn}, the only valid block ordering is: $\X_1 = \{E, F\}$, $\X_2=\{D\}$, with $H^1 = B$.
The compiled class formulas in~\cite{shih_compiling_2019} are \acfp{OBDD}.
\acp{OBDD} correspond to AND-decomposable NNFs when adjusting for notation, and can be efficiently converted to OR-decomposable NNFs of equal size~\cite{jair/DarwicheM21,darwiche_computation_2022}.
\cref{fig:exampleODDnnf} shows a compiled OBDD and an equivalent OR-decomposable NNF.

\begin{figure}
\centering
\begin{subfigure}{0.48\textwidth}
\centering
\digraph[scale=0.2]{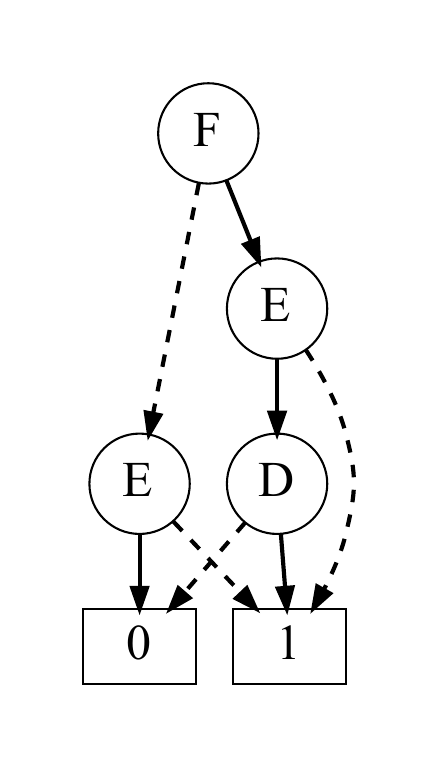}{
	0 [label=F shape=circle fontsize=24]
	0 -> 1 [style=dashed penwidth=2.0]
	1 [label=E shape=circle fontsize=24]
	1 -> S1 [style=dashed penwidth=2.0]
	S1 [label=1 shape=box fontsize=24]
	1 -> S0 [style=solid penwidth=2.0]
	S0 [label=0 shape=box fontsize=24]
	0 -> 6 [style=solid penwidth=2.0]
	6 [label=E shape=circle fontsize=24]
	6 -> S1 [style=dashed penwidth=2.0]
	6 -> 7 [style=solid penwidth=2.0]
	7 [label=D shape=circle fontsize=24]
	7 -> S0 [style=dashed penwidth=2.0]
	7 -> S1 [style=solid penwidth=2.0]}
\caption{Compiled OBDD.}
\end{subfigure}%
\begin{subfigure}{0.48\textwidth}
\centering
\digraph[scale=0.21]{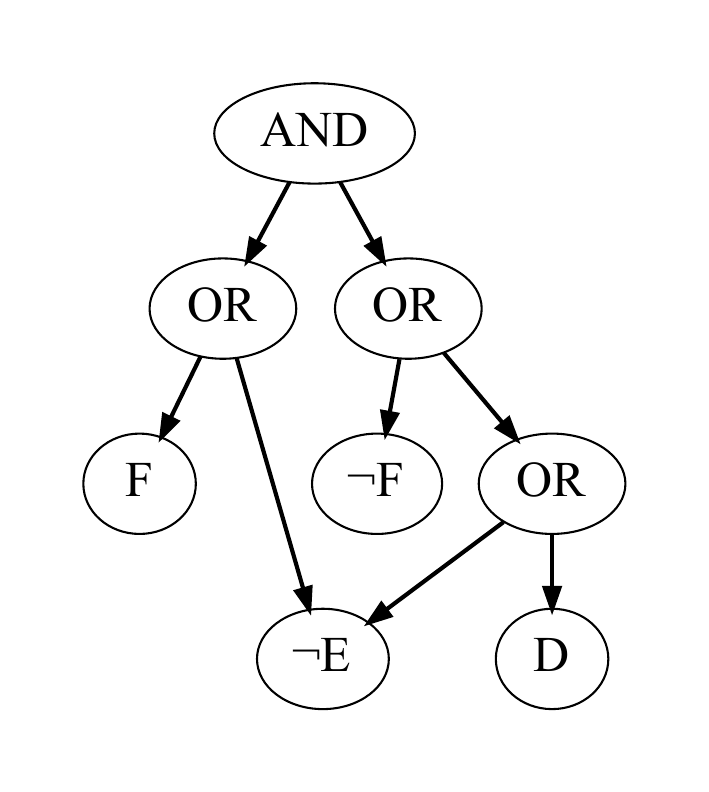}{
	12 [label=AND fontsize=24]
	12 -> 8 [style=solid penwidth=2.0]
	8 [label=OR fontsize=24]
	8 -> 2 [style=solid penwidth=2.0]
	2 [label=F fontsize=24]
	8 -> 4 [style=solid penwidth=2.0]
	4 [label="¬E" fontsize=24]
	12 -> 11 [style=solid penwidth=2.0]
	11 [label=OR fontsize=24]
	11 -> 3 [style=solid penwidth=2.0]
	3 [label="¬F" fontsize=24]
	11 -> 10 [style=solid penwidth=2.0]
	10 [label=OR fontsize=24]
	10 -> 4 [style=solid penwidth=2.0]
	10 -> 7 [style=solid penwidth=2.0]
	7 [label="D" fontsize=24]
}
\caption{Class formula of $y_1$.}
\end{subfigure}%
\caption{(a) An OBDD compiled by~\cite{shih_compiling_2019}. Solid (or dashed) arrows denote feature set to true (or false). (b) An equivalent OR-decomposable NNF.}
\label{fig:exampleODDnnf}
\end{figure}

\section{Feature Tree}
\label{sec:ftree}

To construct an OR-decomposable \ac{NNF} for the formula of class $y_i$ (\cref{fig:nnf2}), we will first construct an AND-decomposable \ac{NNF} (\cref{fig:nnf1}) representing all the negative instances of $y_i$, and then negate the \ac{NNF}.
Instead of a linear block ordering of features, we will arrange features using a feature tree (f-tree), which is a sub-tree embedded in a jointree. 
For some edge $e$ in the f-tree with separator $\Sep$, we will introduce the notion of feature-splitting edge, which is an edge $e$ that partitions features $\X$ into $(\U, \V)$ where $\U \cup \{Y\}$ are d-separated from $\V$ by $\Sep$.
We will then show the key condition that enables the algorithm to sometimes classify a partial feature instantiation $\uu$ of $\U$ without considering $\V$ at all. 
As the f-tree is traversed during the compilation, features are instantiated to form a partial instantiation $\uu$, and the algorithm uses this key condition to check if $\uu$ can be classified. 
If so, for this $\uu$, it does not enumerate over the remaining features $\V$, and thus avoids redundant computation.
Otherwise, it continues the traversal to build an NNF over $\V$, such that the conjunction of $\uu$ and the built NNF represents the negative instances for the target class $y_i$. 

We will first introduce the f-tree in this section, and then describe the compilation algorithm in the next section.

\begin{definition}[Feature tree (f-tree)]
    For a \acl{BNC} with network $\mathcal{N}$, feature variables $\X$, and target variable $Y$, 
    let $\mathcal{T}^J$ be a jointree for $\mathcal{N}$.
    A \textbf{feature tree} $\mathcal{T}$ for the \ac{BNC} is a connected sub-tree of $\mathcal{T}^J$, such that each variable in $\X \cup \{Y\}$ must appear in some cluster in $\mathcal{T}$.
\end{definition}

\cref{fig:ftree} shows a jointree and an embedded f-tree for the network shown in \cref{fig:examplebn}.
Note that an f-tree does not need to include every jointree node that mentions $\X \cup \{Y\}$.
Since the cluster of Node 6 includes target $A$, this f-tree does not need to include Node 5.
For conciseness of our discussion and without loss of generality, we will assume that, features only appear in leaf nodes of the f-tree, and the f-tree is binary.\footnote{Leaf variables (features) can be moved from some internal node $n$ to a newly created leaf node in a jointree; the jointree width is not increased because the newly created node has a cluster no greater than what node $n$ had.
While preserving width, a jointree can also be made binary, i.e., no node has more than three neighbors~\cite{darwiche_modeling_2009}. A connected sub-tree of a binary jointree is also binary.}

\begin{figure}
\centering
\digraph[scale=0.25]{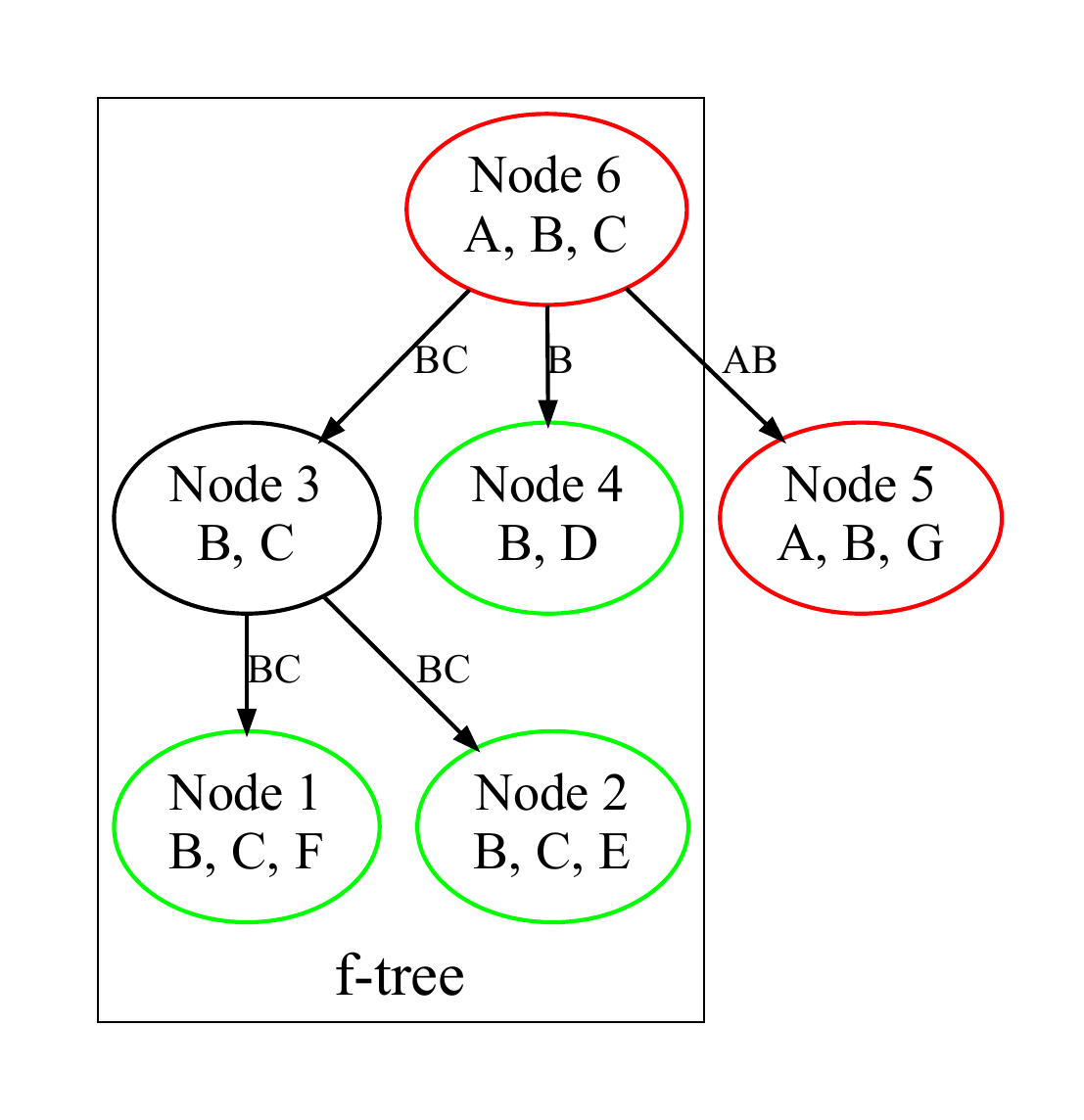}{
    fontsize=30
	ABG [label="Node 5
A, B, G" color=red fontsize=26 penwidth=2]
    subgraph cluster_0 {

        label="f-tree"
        labelloc=b
    ABC [label="Node 6
A, B, C" color=red fontsize=26 penwidth=2]
    BCF [label="Node 1
B, C, F" color=green fontsize=26 penwidth=2]
	BC [label="Node 3
B, C" fontsize=18 fontsize=26 penwidth=2]
    BCE [label="Node 2
B, C, E" color=green fontsize=26 penwidth=2]
	BD [label="Node 4
B, D" color=green fontsize=26 penwidth=2]
}
	ABC -> BD [label=B color=Black fontsize=20 penwidth=2]
	ABC -> BC [label="BC" color=Black fontsize=20 penwidth=2]
	ABC -> ABG [label="AB" color=Black fontsize=20 penwidth=2]
	BC -> BCF [label="BC" color=Black fontsize=20 penwidth=2]
	BC -> BCE [label="BC" color=Black fontsize=20 penwidth=2]
}
\caption{A jointree and an f-tree (the sub-tree in the box) for the network in \cref{fig:examplebn}. 
The nodes are labeled by their clusters and the edges are labeled by their separators. 
The nodes that include the target variable A are in red, and the nodes that include the features are in green.}
\label{fig:ftree}
\end{figure}

\subsection{Splitting Features}

Since an f-tree is a sub-tree of a jointree, it follows that if some variable appears in two clusters, the variable must appear in every cluster and every separator on the path connecting the two clusters.
In other words, if a variable does not appear in the separator of an edge, the variable must only appear on one side of the edge.
Also, recall that an edge's separator d-separates the variables that only appear on one side from the variables that only appear on the other side.

\begin{definition}[Feature-Splitting Edge]\label{def:feaSplitE}
    For an f-tree $\mathcal{T}$ of a \ac{BNC} $(\mathcal{N}, \X, Y)$, let $e$ be an edge in $\mathcal{T}$, and $\Sep$ be the separator for $e$.
    Edge $e$ is a feature-splitting edge iff $\Sep$ does not include any variable in $\{Y\} \cup \X$.
    Let $\U$ be the features that appear on the same side of the edge as $Y$, and let $\V = \X \setminus \U$ be the features on the other side.
    We will then say that edge $e$ splits features $\X$ into partition $(\U, \V)$.
\end{definition}

A feature-splitting edge $e$ has the key property that, its separator $\Sep$ d-separates $\{Y\} \cup \U$ from $\V$. 
Let $\sep$ denote an instantiation of $\Sep$.
The joint distribution $P(Y, \X)$ can be re-written as:

\begin{equation}\label{eqn:factorize}
    P(Y, \X) = \sum_{\sep} P(Y, \U, \sep) P(\V|\sep)
\end{equation}

With the f-tree in \cref{fig:ftree}, the distribution $P(A, D, E, F)$ can be factorized, e.g., by considering the edge $e$ between node 3 and 6. 
Let $\Sep = \{B, C\}$, $\U = \{D\}$, and $\V = \{E, F\}$, we have $P(A, D, E, F) = \sum_{bc} P(A, D, bc) \cdot P(E, F | bc)$, where $bc$ denotes an instantiation of $\{B, C\}$.

\subsection{Classifying Partial Feature Instantiation}

With features $\X$ partitioned into $(\U, \V)$, an instantiation $\uu$ of $\U$ is a partial feature instantiation.
We will call $\uu$ a \textit{positive partial instantiation} for $y_i$ iff $\x = (\uu, \vv)$ is classified as $y_i$ for all possible instantiation $\vv$ of $\V$.
Conversely, the partial instantiation $\uu$ is negative for $y_i$ iff $(\uu, \vv)$ is not classified as $y_i$ for all possible $\vv$.

Recall that we want to compile an AND-decomposable NNF for the negative instances for a target class $y_i$.
Instance $\x$ is negative iff some class $y_j \neq y_i$ is more likely than $y_i$, i.e., $P(y_i, \x) - P(y_j, \x) \le 0$.
For edge $e$ with separator $\Sep$ that splits the instance $\x$ into $(\uu, \vv)$, we can re-write the above condition using the factorization shown in \cref{eqn:factorize}:

\begin{equation}
\begin{gathered}\label{eqn:criteria}
P(y_i, \x) - P(y_j, \x)=
    \sum_{\sep} \Delta_{\sep}^{ij}(\uu) \cdot P(\vv | {\sep}) \le 0, \\
    \text{where $\Delta_{\sep}^{ij}(\uu) = P(y_i, \uu, {\sep}) - P(y_j, \uu, {\sep})$}
\end{gathered}
\end{equation}

Let $\dom(\V)$ denote the instantiations ${\bf v}$ of variables ${\V}$ (i.e., state space), and $|\dom(\cdot)|$ denote the size of that domain.
$\Delta_\Sep^{ij}(\uu)$ can be viewed as an array of size $|\dom(\Sep)|$.
If every element $\Delta_{\bf s}^{ij}(\uu)$ of $\Delta_\Sep^{ij}(\uu)$ is positive (or non-positive), we will say that the array $\Delta_\Sep^{ij}(\uu)$ is positive (or non-positive). We now have the following immediate result.

\begin{theorem}\label{thm:uClassify}
    Let $\mathcal{T}$ be an f-tree of \ac{BNC} $(\mathcal{N}, \X, Y)$, $e$ be an edge in $\mathcal{T}$ that splits features $\X$ into partition $(\U,\V)$.
    For a partial feature instantiation $\uu$, if $\Delta_\Sep^{ij}(\uu)$ is non-positive for some $j \neq i$, $\uu$ is negative for class $y_i$.
    Conversely, if $\Delta_\Sep^{ij}(\uu)$ is positive for every $j \neq i$, $\uu$ is positive for class $y_i$.
\end{theorem}

In our example f-tree \cref{fig:ftree}, suppose $\uu = d_0$, then we have $\Delta_\Sep^{ij}(\uu) = P(y_i, d_0, B, C) - P(y_j, d_0, B, C)$.
Crucially, $\Delta_\Sep^{ij}$ is a factor only over the separator $\{B, C\}$ and does not depend on the values of features $\V = \{E, F\}$.

We now have the key component (\cref{alg:decideNeg}) that will be used by our compilation algorithm.
Whenever the partial instantiation $\uu$ is updated, we compute a new factor $M= P(Y, \uu, \Sep)$, and call \Call{decide-negative}{$M, \Sep, y_i$} to check whether $\uu$ can be decided to be either positive or negative for class $y_i$.
\Cref{line:feaSplitE} checks whether the edge with separator $\Sep$ is a feature-splitting edge by \cref{def:feaSplitE};
if so, \cref{thm:uClassify} is directly applied.
Note that when $\Delta_\Sep^{ij}$ is computed in \cref{line:Deltaij}, factor $M$ is treated as a matrix of shape $|\dom (Y)| \times |\dom(\Sep)|$, and $M[y_i]$ denotes an array of size $|\dom(\Sep)|$ that represents $P(y_i, \uu, \Sep)$.

\begin{algorithm}
\caption{\Call{decide-negative}{$M, \Sep, y_i$}}
\label{alg:decideNeg}
\begin{algorithmic}[1]

\Require{$M=P(Y, \uu, \Sep)$}
\Ensure{If partial instantiation $\uu$ cannot be decided, return \nil; otherwise, return whether $\uu$ is \textit{negative} for $y_i$.}
\If{$\Sep \cap (\X \cup \{Y\})$ is $\emptyset$} \label{line:feaSplitE}
    \State $b \gets \true$
    \ForAll{$y_j \in Y$, ($y_j \neq y_i$)}
        \State $\Delta^{ij}_\Sep \gets M[y_i] - M[y_j]$ \label{line:Deltaij}
        \If{$\Delta^{ij}_\Sep$ is non-positive} \Return{$\true$} \label{line:uNeg} \EndIf
        \State $b \gets b \mbox{ and } (\Delta^{ij}_\Sep$ is positive$)$
    \EndFor
    \If{$b$} \Return{\false} \label{line:uPos} \EndIf 
\EndIf
\State \Return \nil
\end{algorithmic}
\end{algorithm}

\section{Compiling Class Formulas using f-Trees}
\label{sec:algMain}

We will begin our discussion of the compilation algorithm with some notation.
For a (binary) f-tree $\mathcal{T}$, a node that contains the target $Y$ in its cluster is selected as the root node.
For some node $q$ in $\mathcal{T}$, the neighbor closer to the root is the parent node $p$, and the neighbors further away from the root are the children nodes $(l, r)$.\footnote{We will show later that having only one child node can be treated as a trivial special case.}
Let $\V$ denote the features that appear in the sub-tree rooted by $q$, and $\U = \X \setminus \V$ be the rest of the features. 
Let $\V^l, \V^r$ denote the features that appear in the $l$ and $r$ sub-trees respectively. 
Since features only appear in leaf nodes of an f-tree, $\V^l \cap \V^R = \emptyset$.

\begin{algorithm}

\caption{\Call{compile-negative}{$q, M$}\\
$y_i$: the target class, 
$\phi_n$: marginal over the cluster of node $n$
}
\label{alg:compile_classifier}

\begin{algorithmic}[1]

\State $\Sep_{pq} \gets$ separator between $q$ and its parent $p$
\State $d \gets \Call{decide-negative}{M, \Sep_{pq}, y_i}$ \label{line:callDecide}
\If {$d$ is not \nil} \Return{$d$} \EndIf

\State $\Gamma \gets \emptyset$
\If{$q$ is a leaf node}
    \ForAll{ $(\vv, \pr) \in \Call{get-instances}{q}$}:
        \State $\psi \gets \Call{project}{M  \cdot \pr,Y}$\Comment{$\psi = P(Y, \x)$}
        \If{$\arg\max \psi$ is not $y_i$}
        add $\vv$ to $\Gamma$ \label{line:leafArgmax}
        \EndIf
\EndFor

\Else
    \State $\phi_q' \gets M \cdot \phi_q$ \label{line:compileMult}
    \State $\Sep_{qr} \gets $ separator between $q$ and child $r$
    \State $\phi_{rq} \gets \Call{project}{\phi_{r}, \Sep_{qr}}$
    \ForAll{ $(\vv^l, \pr^l) \in \Call{get-instances}{l}$} \label{line:enumL}
        \State $M_{qr} \gets \Call{project}{\phi_q' \cdot \pr^l, \Sep_{qr} \cup \{Y\}} / \phi_{rq}$ \label{line:newMsg}
        \State $\alpha \gets \Call{compile-negative}{r, M_{qr}}$ \label{line:compileRecur}
        \State add $(\vv^l \land \alpha)$ to $\Gamma$
    \EndFor
\EndIf
\If {$\Gamma$ is $\emptyset$} \Return{\false} \EndIf
\State \Return $\Call{disjoin}{\Gamma}$
\end{algorithmic}
\end{algorithm}

The main compilation algorithm (\cref{alg:compile_classifier}) is a recursive procedure \Call{compile-negative}{$q, M$} on f-tree node $q$ by its parent $p$.
The argument $M$ is a factor $P(Y, \uu, S_{pq})$, where $\uu$ is a distinct instantiation of features $\U$, and $S_{pq}$ is the separator of the edge $(p,q)$.
On \cref{line:callDecide}, the algorithm calls \Call{decide-negative}{$\cdot$} to check whether $\uu$ can be classified using \cref{thm:uClassify}; if so, it returns immediately.

In the base case that $q$ is a leaf node, for each $\vv$ of $\V$, we can compute $P(Y, \uu, \vv)$, and if $\x = (\uu, \vv)$ is negative for the target class $y_i$, $\vv$ is added to the set $\Gamma$.
If $q$ is an internal node, the algorithm enumerates over the instantiations of features $\V^l$ from child $l$.
For each instantiation $\vv^l$, the algorithm computes a factor $M_{qr} = P(Y, \uu, \vv^l, \Sep_{qr})$ (\cref{line:newMsg}), and makes a recursive call on the child $r$ (\cref{line:compileRecur}).
The return value $\alpha$ is an AND-decomposable NNF over $\V^r$, such that for each $\vv^r$ that is a model of $\alpha$, $(\uu, \vv^l, \vv^r)$ is a negative instance for $y_i$.
Since $\V^l \cap \V^r = \emptyset$, $\vv^l \land \alpha$ is also an AND-decomposable NNF, and is added to the set $\Gamma$.
Finally, the algorithm returns a disjunction over $\Gamma$, which is still AND-decomposable.

The helper function \Call{get-instances}{$q$} enumerates over instantiation $\vv$ of $\V$ (the features in the $q$ sub-tree), and returns all possible $(\vv, P(\vv|\Sep_{pq}))$, where $\Sep_{pq}$ is the separator between node $q$ and its parent.
Since this function does not depend on $\uu$, the results are cached and reused.
When node $q$ only has one child $r$ and $l = \nil$, \Call{get-instances}{$l$} (\cref{line:enumL}) returns just a single tuple with $\vv^l=\top$ and $p^l = 1$.
The \Call{project}{$\cdot$} function is the standard factor projection. 
This discussion leads to the following result.

\begin{theorem}
    \cref{alg:compile_classifier} returns an AND-decomposable \ac{NNF} representing the negative instances for class $y_i$.
\end{theorem}

\subsection{Complexity Analysis}

We now consider the time complexity of the algorithm.
Let $n$ be the number of variables in the network, and $\omega$ be the width of the jointree.
Let $C_q$ be the cluster of f-tree node $q$, and \ftw{$q$} be the width of the sub-tree (rooted at $q$) of the f-tree.
Features inside and outside of the sub-tree are denoted by $\V_q$ and $\U_q$ respectively.
Let $R$ be the nodes on the rightmost path of an f-tree $\mathcal{T}$, i.e., every node is the right (or only) child of its parent, and let $L$ denote the left child of the nodes in $R$.
The width of the compilation algorithm is $\omega^\mathcal{T} = \max(\omega, \omega^R , \omega^L)$, where 
$\omega^R = \max_{q \in R} |\U_q| + |C_q|$, and
$\omega^L = \max_{q \in L} |\V_q| + \ftw{q}.$

\begin{theorem}
   The time complexity of \cref{alg:compile_classifier} is $\mathcal{O}(n \exp \omega^\mathcal{T}),$ 
   where $n$ is the number of variables in the network, and $\omega^\mathcal{T}$ is the width of the compilation algorithm.
\end{theorem}

\begin{proof}

The algorithm needs to first compute the cluster marginals for the f-tree, which takes time $\mathcal{O}(n \exp{\omega})$.
For the main compilation process, all calls to \Call{compile-negative}{$q, \dots$} have $q \in R$; also, all calls to \Call{get-instances}{$q$} have $q \in L$.
In the worst case, the algorithm is never able to classify a partial feature instantiation, it needs to compute $p(Y, \x)$ for every possible $\x$ (in the base case of the recursion).\footnote{Each full instantiation $\x$ is examined at most once, so the algorithm is guaranteed to terminate.}
So, for each node $q \in R$, there are $\mathcal{O}(\exp|\U_q|)$ calls to \Call{compile-negative}{$\cdot$}; each call requires computing factor over cluster $C_q$ and takes time $\mathcal{O}({\exp |C_q|})$.
Thus, the overall time for \Call{compile-negative}{$\cdot$} is $\mathcal{O}(n \exp{\omega^R})$.
Finally, for each $q \in L$, the initial call to \Call{get-instances}{$q$} requires time $\mathcal{O}(\exp(|\V_q| + \ftw{q})$; the output is cached and reused in subsequent calls.
The total time for \Call{get-instances}{$\cdot$} is then $\mathcal{O}(n \exp{\omega^L})$.
In total, the algorithm takes time $\mathcal{O} (n (\exp{\omega} + \exp{\omega^R} + \exp{\omega^L}))$, which can be simplified to $\mathcal{O}(n \exp \omega^\mathcal{T})$.
\end{proof}

It is worth noting that the total runtime theoretically cannot be lower than $\mathcal{O}(n \exp{\omega})$,
which is the baseline complexity for probabilistic inference in the Bayesian network.

For the brute force method of enumerating over all feature instances and making an inference call for each one, the complexity will be $\mathcal{O}(n \exp{(\omega + |\X|}))$.
For any node $q$ of any possible f-tree, we have $|\U_q| \le |\X|$, $|\V_q| \le |\X|$, and $|C_q| \le \ftw{q} \le \omega$, so $\omega^\mathcal{T} \le \omega + |\X|$ always holds. 
Later in \cref{sec:results}, we will show a comparison of $\omega^\mathcal{T}$ and $\omega + |\X|$ on some \ac{BNC}s.
As for the size of the compiled NNF, since there are $\exp(|\X|)$ feature instances, and each one can be represented with $\mathcal{O}(|\X|)$ nodes, the number of nodes is bound by $\mathcal{O}(|\X| \cdot \exp{(|\X|)})$.

\section{Experimental Results}
\label{sec:results}

The proposed compilation algorithm is evaluated in the following ways.
We first evaluate the theoretical time complexity by comparing the the compilation width of the proposed algorithm against the brute force method.
Then, we compare the actual runtime of the proposed algorithm against the prior work in~\cite{shih_compiling_2019}, which as far as we are aware, introduced the only existing compilation algorithm for a \ac{BNC} with arbitrary structure.
We will refer to this prior work as the baseline algorithm.
We also examine the compilation time and circuit size for a few very hard problems reported by this prior work.

The five Bayesian networks used to evaluate the baseline algorithm in the prior work are also used in our evaluation:
network \texttt{win95pts} with 76 nodes, 16 leaf variables, 34 roots, and width 9;
network \texttt{cpcs54} with 54 nodes, 12 leaf variables, 13 roots, and width 14;
network \texttt{andes} with 223 nodes, 22 leaf variables, 86 roots, and width 18;
network \texttt{tcc4e} with 98 nodes, 60 leaf variables, 35 roots, and width 10;
network \texttt{emdec6g} has 168 nodes, 117 leaf variables, 46 roots, and width 7.
These widths are approximated by the minfill heuristic for constructing jointrees~\cite{kjaerulff_triangulation_1990}.
The baseline algorithm cannot solve problems with multi-class classifiers so all five networks have binary target variables.
Note that when target $Y$ is binary, a threshold $t$ can be specified, and $\x$ is classified as class $y_i$ iff $P(y_i|\x) > t$.

We created a benchmark of 554 problems to test the algorithms, where each problem corresponds to a network, a set of features, a target variable, and a threshold.
An overview of the benchmark problems is shown in \cref{tab:probStats}.
For networks \texttt{win95pts} and \texttt{cpcs54}, all leaf variables are used as features.
For \texttt{andes}, half of the problems samples 19 out of 22 leaves as features, and the other half uses all leaves.
All problems with \texttt{tcc4e} and \texttt{emdec6g} uses a subset of leaves as features.
With the same network, target variable, and features, there are two problems with different thresholds.
One is set to $0.5$, which is used in prior work.
The other threshold is the posterior $P(y_i | \x)$ averaged over all $\x$ of $\X$, which typically makes the threshold more balanced.

\begin{table}
\centering
\begin{tabular}{lcc}
\toprule
network & \# features & \# unique target $Y$ \\
\midrule
win95pts & 16 & 33 \\
cpcs54 & 13 & 12 \\
andes & 19 & 86 \\
andes & 22 & 86 \\
tcc4e & 27 & 16 \\
tcc4e & 30 & 15 \\
emdec6g & 29 & 29 \\
\bottomrule
\end{tabular}
\caption{Overview of the 554 benchmark problems. With the same network, target, and features, there are two problems with different thresholds.}
\label{tab:probStats}
\end{table}

We first conducted an experiment to compare the time complexity of the proposed algorithm against the brute force method.
We chose the three networks with problem settings that use all leaves as features.
The average compilation width $\omega^\mathcal{T}$ for the three networks are shown in \cref{tab:width}.
For all three networks, the differences between the two widths ($\omega^\mathcal{T}$ and $|\X| + \omega$) are significant, since time complexity is exponential with respect to the width.

\begin{table}
\centering
\begin{tabular}{lcccc}
\hline
network & avg $\omega^\mathcal{T}$ & width $\omega$ & $|\X|$ & $|\X| + \omega$\\
\hline
win95 & 17.1 & 9 & 16 & 25\\
cpcs54 & 18.1 & 14 & 13 & 27\\
andes & 27.7 & 18 & 22 & 40\\
\hline
\end{tabular}
\caption{Compilation width compared to brute force method. For all networks, all leaves are chosen as features, and $|\X|$ denotes the number of features. 
For each network, the compilation width $\omega^\mathcal{T}$ is averaged over the the set of targets used in the benchmark. 
Column $|\X| + \omega$ shows the width for brute force method. }
\label{tab:width}
\end{table}

Next, we ran the benchmark using the released Java implementation of the baseline algorithm,\footnote{\url{https://github.com/AndyShih12/BNC_SDD}} and our Python implementation of the proposed algorithm.\footnote{\url{https://github.com/victorzhangyf/bnc2d}}
The benchmark is ran with an Intel Xeon E5-2670 CPU, and each problem has a timeout limit of 4 hours.
Out of a total of 554 problems, the baseline algorithm completed 166 problems, and the proposed algorithm completed 492 problems.
\cref{fig:cactusTime} shows the runtime performance of the two algorithms. 
The cactus plot is created by first sorting the runtime of the completed problems by an algorithm in an increasing order, and then the cumulative runtime is computed and plotted.
The plot shows that the proposed algorithm completes the same number of problems (e.g., 100) with much less time.
Also, as the problems get more difficult, the cumulative time of the baseline algorithm grows at a much steeper rate compared to the proposed algorithm. 

\begin{figure}
\centering
\includegraphics[scale=0.5]{"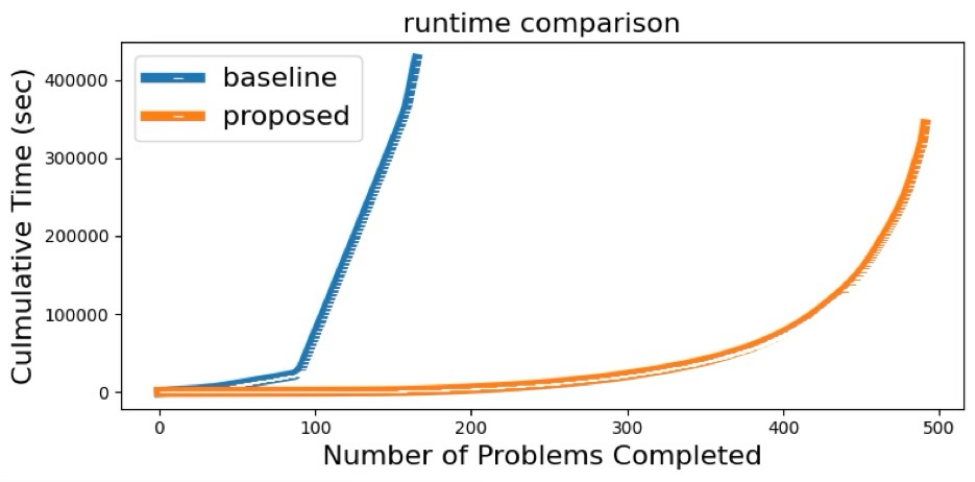"}%
\caption{Cactus plot of runtime.}%
\label{fig:cactusTime}
\end{figure}

Finally, we will compare the runtime and circuit size for the four problems in the \texttt{andes} network that are reported in~\cite{shih_compiling_2019}. 
A comparison is shown in \cref{tab:andesCompare}.
The proposed algorithm solves these problems significantly faster.
The sizes of the compiled circuits are quite comparable.
Since the baseline algorithm takes quite long to solve these relatively hard problems, it is difficult to collect data for a more comprehensive size comparison.
It is worth noting that, for both algorithms, the circuit sizes here are significantly smaller than the number of feature instances. 
For example, for CNBG, the target class includes 1,719,892 instances, out of a total of 4,194,304 possible instances (from 22 binary features).

\begin{table}
\centering
\begin{tabular}{l|cc|cc}
\hline
& \multicolumn{2}{c|}{Time (s)} & \multicolumn{2}{c}{Size (\# Nodes)}\\
Y & baseline & proposed  & baseline & proposed \\
\hline
TK & 11,708 & 12  & 47 & 5 \\
CNBG & 24,374 & 30 & 2,893 & 2,453\\
MDA & 26,614 & 1936  & 5,454 & 2,894\\
VKE & 27,495 & 262  & 2,107 & 1,725\\
\hline
\end{tabular}
\caption{Time and size comparison for the 4 roots ($Y$) and all leaves in the \texttt{andes} network. Threshold is $0.5$ for all problems.}
\label{tab:andesCompare}
\end{table}

\section{Explanations in a Multi-Class \ac{BNC}}
\label{sec:multiclass}

We will next illustrate how a decision of a multi-class BNC can be explained using the output of the compilation algorithm.
Consider a Bayesian network (\Cref{fig:discreteExplainExample}) about diagnosing the type of some disease.
The root variable (D) represents the disease type, and it is a discrete variable with 3 possible values $\{0, 1, 2\}$.
The two binary variables $X$ and $Y$ each represent the presence of some condition caused by the disease.
The leaf variable $\CT$ is binary, and represents whether the result of some CT scan is $\lnot \CT$ (negative) or $\CT$ (positive).
The leaf variables heart rate ($\HR$) and blood pressure ($\BP$) are both discrete, and each has three possible values $\{\vLow, \vNormal, \vHigh\}$.
This \ac{BNC} has in total 18 possible instances, and 3 possible classes.
A literal of a discrete variable can be generally denoted as a set, e.g., $\HR \in \{\vLow, \vNormal\}$ restricts the variable $\HR$ to be either $\vLow$ or $\vNormal$.
When that set contains just one state (say, $\vLow$), we will simply denote it as $\HR=\vLow$.

Suppose a patient (instance $\I$) has high blood pressure, a positive CT scan, and normal heart rate; i.e., $\I = [(\BP=\vHigh) \land \CT \land (\HR=\vNormal)]$, and the patient is diagnosed (decided) to have disease type 2. 
The class formula of type 2 ($\Delta_2$) is represented by the OR-decomposable NNF shown in \Cref{fig:nnfType2}.
We may now ask: What is the reason behind this decision?
\begin{figure}
\centering
\begin{subfigure}{0.48\columnwidth}
\digraph[scale=0.35]{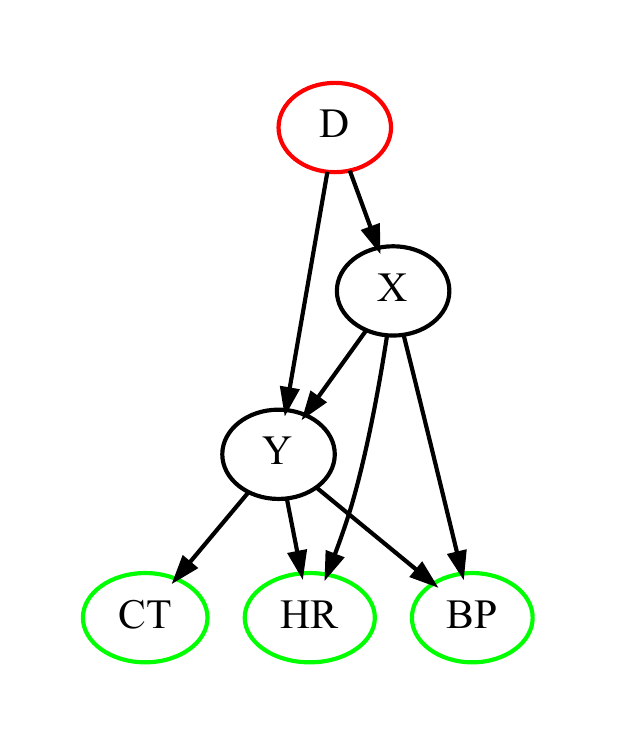}{
"D" [color=red fontsize=20 penwidth=2.0];
"X" [fontsize=20 penwidth=2.0];
"Y" [fontsize=20 penwidth=2.0];
"CT" [color=green fontsize=20 penwidth=2.0];
"BP" [color=green fontsize=20 penwidth=2.0];
"HR" [color=green fontsize=20 penwidth=2.0];
"D" -> "X" [style=solid penwidth=2.0] 
"D" -> "Y" [style=solid penwidth=2.0]
"X" -> "Y" [style=solid penwidth=2.0]
"X" -> "BP" [style=solid penwidth=2.0]
"X" -> "HR" [style=solid penwidth=2.0]
"Y" -> "CT" [style=solid penwidth=2.0]
"Y" -> "BP" [style=solid penwidth=2.0]
"Y" -> "HR" [style=solid penwidth=2.0]
}
\caption{Network structure. }
\label{fig:discreteExplainExample}
\end{subfigure}%
\begin{subfigure}{0.48\columnwidth}
\centering
\digraph[scale=0.25]{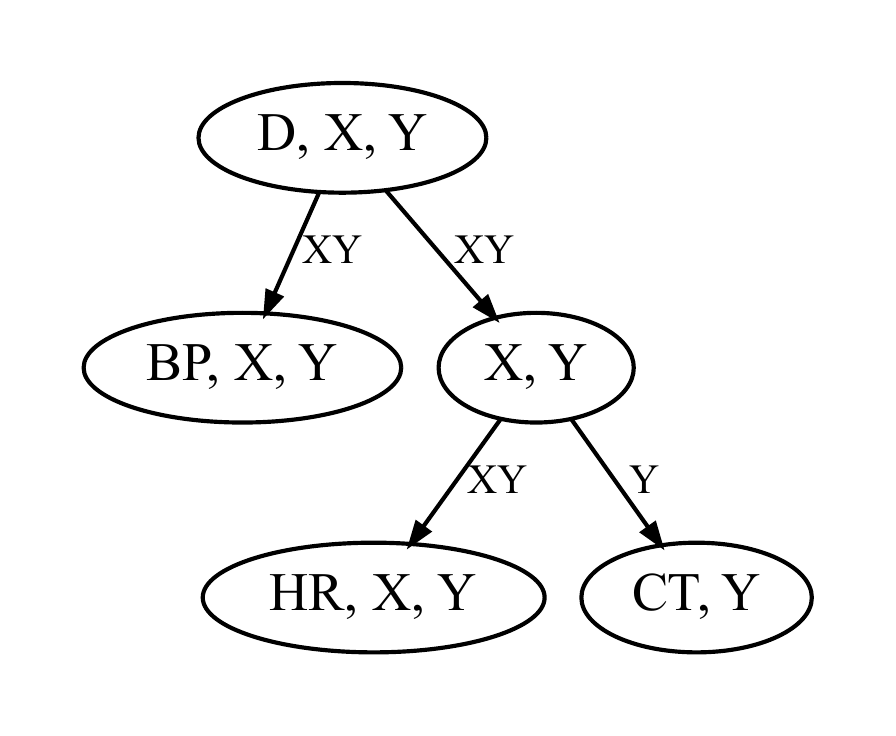}{
    fontsize=30
    DXY [label="D, X, Y" fontsize=26 penwidth=2]
    BPXY [label="BP, X, Y" fontsize=26 penwidth=2]
	XY [label="X, Y" fontsize=18 fontsize=26 penwidth=2]
    HRXY [label="HR, X, Y" fontsize=26 penwidth=2]
	CTY [label="CT, Y" fontsize=26 penwidth=2]

	DXY -> BPXY [label="XY" color=Black fontsize=20 penwidth=2]
	DXY -> XY [label="XY" color=Black fontsize=20 penwidth=2]
	XY -> HRXY [label="XY" color=Black fontsize=20 penwidth=2]
	XY -> CTY [label="Y" color=Black fontsize=20 penwidth=2]
}
\caption{An f-tree of the network.}
\label{fig:ftree2}
\end{subfigure}
\caption{The network structure and f-tree of a \ac{BNC} about the diagnosis of disease. Note that the f-tree shown in (b) is also a jointree of the network.}
\end{figure}

A reason for a decision can be viewed as a condition (i.e., a logical formula) that must be satisfied by the instance being decided, and this condition must also imply the class formula of the decided class~\cite{lics/Darwiche23}.
While there are generally many logical formulas that qualify as reasons for a decision, some reasons can be more informative than others. For example, for any decision, the instance itself (i.e., a term) and the class formula both technically qualify as reasons, but they are generally not informative.

\begin{figure}
\centering
\digraph[scale=0.3]{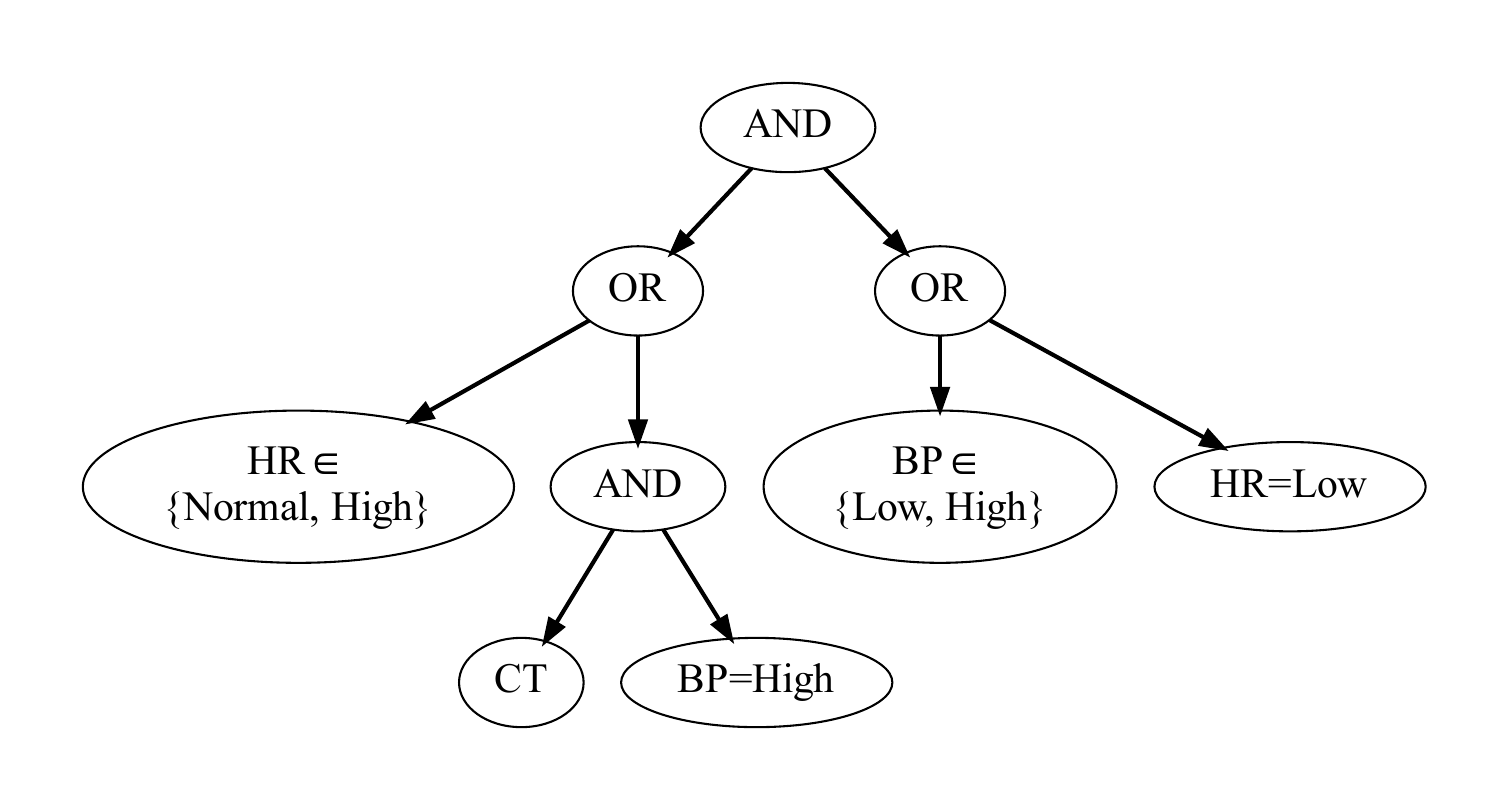}{
	1 [label=AND fontsize=20]
	1 -> 2 [style=solid penwidth=2.0]
	1 -> 7 [style=solid penwidth=2.0]
	2 [label=OR fontsize=20]
	2 -> 3 [style=solid penwidth=2.0]
	2 -> 4 [style=solid penwidth=2.0]
	3 [label="HR∈
    {Normal, High}" fontsize=20]
	4 [label="AND" fontsize=20]
    4 -> 5 [style=solid penwidth=2.0]
	4 -> 6 [style=solid penwidth=2.0]
	5 [label="CT" fontsize=20]
	6 [label="BP=High" fontsize=20]
    7 [label=OR fontsize=20]
	7 -> 8 [style=solid penwidth=2.0]
	7 -> 9 [style=solid penwidth=2.0]
	8 [label="BP∈
    {Low, High}" fontsize=20]
	9 [label="HR=Low" fontsize=20]
}
\caption{An OR-decomposable NNF that represents the class formula of disease type 2.}
\label{fig:nnfType2}
\end{figure}

\subsection{Complete Reasons}
Among the possible reasons for a decision, the \textit{complete reason} is the weakest NNF whose literals appear in the instance being explained~\cite{ecai/DarwicheH20,lics/Darwiche23}.
Since an instance sets each variable to a particular state, each literal that appears in an instance (and by extension in a complete reason) must only contain a single state.
The complete reason can be computed through the \textit{universal literal quantification} operator~\cite{jair/DarwicheM21,darwiche_computation_2022}, which can be evaluated in linear time over an OR-decomposable NNF.
In our example, the complete reason of this patient diagnosed with disease type 2 is

\begin{equation}
    (\BP = \vHigh) \land 
    [(\HR = \vNormal) \lor \CT] 
\label{eqn:cr}
\end{equation}

We could easily verify that every instance that satisfies \cref{eqn:cr} also satisfies the class formula in \cref{fig:nnfType2}.
One usually computes the prime implicants and implicates of the complete reason as these correspond to the notions of sufficient and necessary reasons (SRs and NRs) for the decision, which are very widely used as explanations. 
We will revisit these notions in a more general context later in the section.

We might also observe that if we relax $(\HR = \vNormal)$ to $(\HR \in \{\vNormal, \vHigh\})$ in~\cref{eqn:cr}, every instance that satisfies the resulting (weaker) formula would still belong to type 2, which begs the question:
Would the weaker formula be a better justification for this decision?
As we will show next, with discrete features, by relaxing the restriction posed on the literals in the NNF, we do indeed get an abstraction of the instance that can be more informative.

\subsection{General Reasons}

Among the possible reasons for a decision, the \textit{general reason} is the weakest NNF whose literals are implied by the instance~\cite{ji_new_2023}.
Note that the literal $(\HR \in \{\vNormal, \vHigh\})$ does not appear in the instance, but is implied by the literal $(\HR = \vNormal)$ in the instance.
The general reason of a decision can be computed using the \textit{selection} operator \cite{ji_new_2023}, which also can be evaluated in linear time on an OR-decomposable NNF. 
In our example, the general reason of the patient diagnosed with disease type 2 is: 
\begin{multline*}
(\BP \in \{\vLow, \vHigh\}) \land
[(\HR \in \{\vNormal, \vHigh\}) \lor (\CT \land (\BP = \vHigh))]
\end{multline*}

The general reason can be further used to compute \acfp{GSR} and \acfp{GNR} of the decision.
The \acp{GSR} of a decision are the \textit{variable-minimal} prime implicants of the general reason~\cite{ji_new_2023}.
The two prime implicants of the general reason are 
\begin{align*}
&(\BP \in \{\vLow, \vHigh\}) \land (\HR \in \{\vNormal, \vHigh\})\\
&(\BP = \vHigh) \land \CT
\end{align*}
Both of them are variable-minimal, and are thus \acp{GSR}.
GSRs can be viewed as minimal aspects of the instance which are sufficient
to trigger the decision on that instance.

The \acp{GNR} of a decision are the \textit{variable-minimal} prime implicates of the general reason~\cite{ji_new_2023}.
Each GNR points to a minimal way to modify the instance such that the current decision would no longer hold:
modifying the instance in \textit{any} way that violates a \ac{GNR} would always lead to a different decision.\footnote{In the presence of discrete features, a GNR can be violated in multiple ways, and for every one of those ways, the decision is guaranteed to be changed. We will show this in a later example.}
The three prime implicates of the above general reason are 
\begin{align}
&(\BP \in \{\vLow, \vHigh\}) \label{eqn:GNR2_BP} \\
&[\CT \lor (\HR \in \{\vNormal, \vHigh\})] \label{eqn:GNR2_CT_HR} \\
&[(\BP = \vHigh) \lor (\HR \in \{\vNormal, \vHigh\})] \label{eqn:IP2_BP_HR} 
\end{align}

The first two are variable-minimal, but the last one is not; 
thus, the first two are the \acp{GNR} of the decision.
Violating the GNR in~\cref{eqn:GNR2_BP} requires changing $\BP$ to be $\vNormal$;
violating the one in~\cref{eqn:GNR2_CT_HR} results in changing the values of $\CT$ and $\HR$ to $\lnot \CT$ and $\HR = \vLow$.
One could verify that neither modified instances satisfy $\Delta_2$ (\cref{fig:nnfType2}).

It is worth noting that, in the presence of discrete features, GSRs and GNRs (computed using general reasons) can be more informative than SRs and NRs (computed using complete reasons); see~\cite{ji_new_2023} for details.

\subsection{Contrastive Explanations}

Recall that the \ac{BNC} in our example has three possible classes (disease types 0, 1, 2). 
In a multi-class classifier like this, we may be interested in a more targeted kind of queries; 
for example, why is this patient \textit{not} diagnosed to have disease type 0?
In other words, what is the minimal change to the patient's test results such that the diagnosis is changed from type 2 to type 0?
Here, we are interested in not only changing the current decision, but also targeting the change towards a particular class.\footnote{This distinction does not exist in binary classifiers.}
The answer to such a query is known as a \textit{contrastive explanation}~\cite{Lipton_contrastive_1990}.

\begin{figure}
\centering
\digraph[scale=0.245]{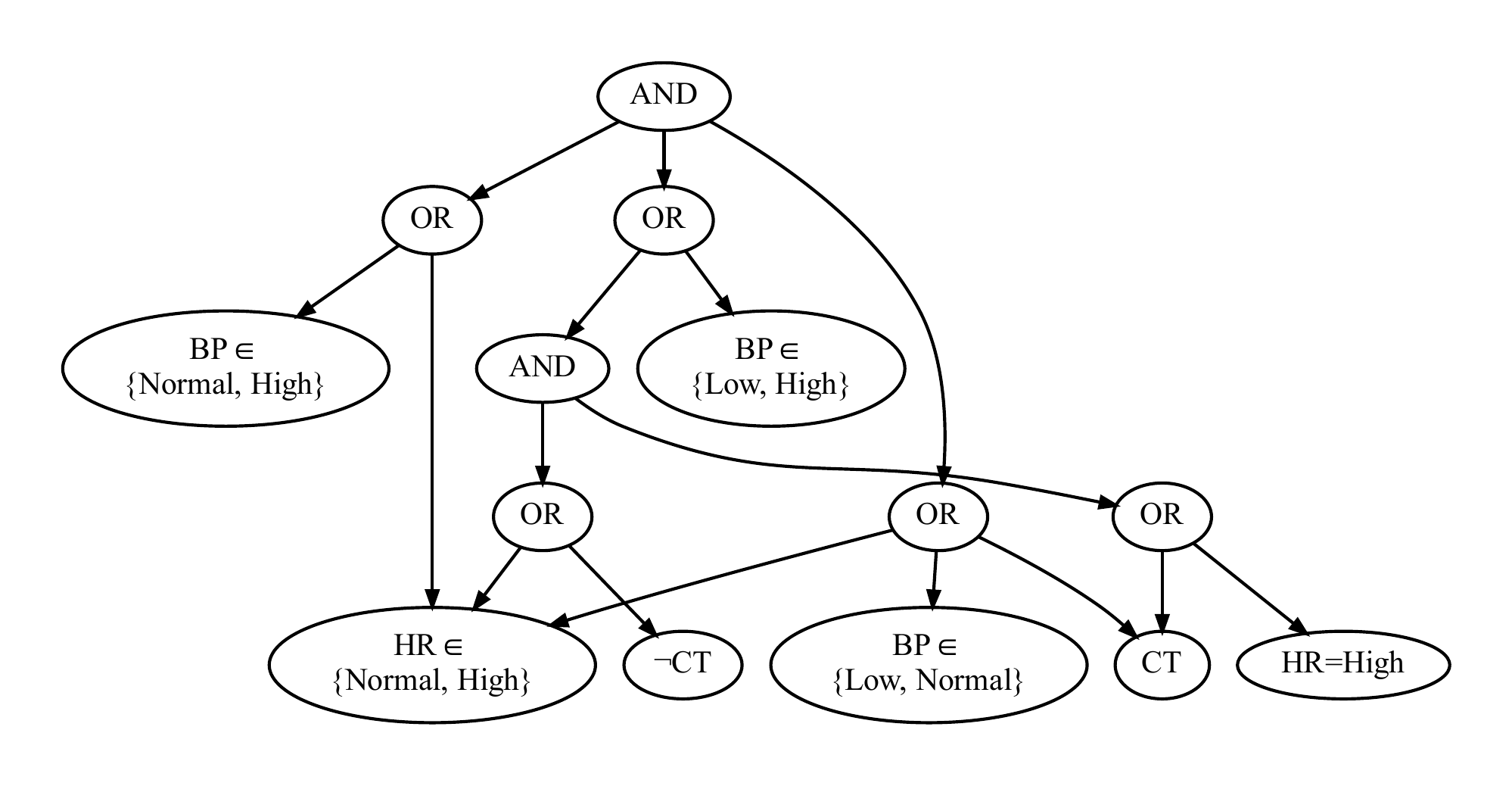}{
	18 [label=AND penwidth=2.0 fontsize=20]
	18 -> 12 [style=solid penwidth=2.0]
	12 [label=OR penwidth=2.0 fontsize=20]
	12 -> 2 [style=solid penwidth=2.0]
	2 [label="BP∈
    {Normal, High}" penwidth=2.0 fontsize=20]
	12 -> 7 [style=solid penwidth=2.0]
	7 [label="HR∈
    {Normal, High}" penwidth=2.0 fontsize=20]
	18 -> 16 [style=solid penwidth=2.0]
	16 [label=OR penwidth=2.0 fontsize=20]
	16 -> 3 [style=solid penwidth=2.0]
	3 [label="BP∈
    {Low, High}" penwidth=2.0 fontsize=20]
	16 -> 15 [style=solid penwidth=2.0]
	15 [label=AND penwidth=2.0 fontsize=20]
	15 -> 11 [style=solid penwidth=2.0]
	11 [label=OR penwidth=2.0 fontsize=20]
	11 -> 6 [style=solid penwidth=2.0]
	6 [label="¬CT" penwidth=2.0 fontsize=20]
	11 -> 7 [style=solid penwidth=2.0]
	15 -> 14 [style=solid penwidth=2.0]
	14 [label=OR penwidth=2.0 fontsize=20]
	14 -> 5 [style=solid penwidth=2.0]
	5 [label="CT" penwidth=2.0 fontsize=20]
	14 -> 13 [style=solid penwidth=2.0]
	13 [label="HR=High" penwidth=2.0 fontsize=20]
	18 -> 17 [style=solid penwidth=2.0]
	17 [label=OR penwidth=2.0 fontsize=20]
	17 -> 4 [style=solid penwidth=2.0]
    4 [label="BP∈
    {Low, Normal}" penwidth=2.0 fontsize=20]
	17 -> 5 [style=solid penwidth=2.0]
	17 -> 7 [style=solid penwidth=2.0]
}
\caption{An OR-decomposable NNF representing $\lnot \Delta_0$.}
\label{fig:nnfTypeNot0}
\end{figure}

Using the approach proposed by \cite{darwiche_computation_2022}, the contrastive explanations in this example can be viewed as the \acp{GNR} for a new class ``not type 0'', which is represented by the class formula $\Delta_{\lnot 0}$, where $\Delta_{\lnot 0} = \Delta_1 \lor \Delta_2 = \lnot \Delta_0$.
By definition, the instance $\I$ must belong to this new class (i.e., $\I$ implies $\Delta_{\lnot 0})$, and any modification to $\I$ that kicks $\I$ out of class ``not type 0'' would change the decision to type 0.\footnote{Computing formula $\Delta_{\lnot i} = \bigvee_{j \neq i} \Delta_j = \lnot \Delta_i$ can be achieved with minor modifications. \cref{alg:decideNeg} can be modified to \Call{decide-positive}{} by changing \cref{line:uNeg} to return \false, and \cref{line:uPos} to return \true. 
\cref{alg:compile_classifier} can be changed to \Call{compile-positive}{} by flipping the condition on \cref{line:leafArgmax} to
$\arg\max \psi$ is $y_i$.}

The OR-decomposable NNF shown in \cref{fig:nnfTypeNot0} represents the class formula $\Delta_{\lnot 0}$.
The general reason for \textit{not} decided as type 0 can be written in the prime implicate form as,

\begin{align*}
&[\CT \lor (\HR \in \{\vNormal, \vHigh\})] \land \\ 
&[(\BP = \vHigh) \lor (\HR \in \{\vNormal, \vHigh\})] \land \\
&[(\BP \in \{\vLow, \vHigh\}) \lor \CT] 
\end{align*}

The three prime implicates are all variable-minimal so they are all GNRs in this context.
If the instance is modified in any way that violates one of these GNRs, the diagnosis would change from type 2 to 0.
For example, to violate $[(\BP = \vHigh) \lor (\HR \in \{\vNormal, \vHigh\})]$, $\HR$ must be $\vLow$, and $\BP$ could be either $\vLow$ or $\vNormal$; thus, there are two ways to violate this GNR.
One could verify that neither modified instances satisfy $\lnot \Delta_0$ (\cref{fig:nnfTypeNot0}); that is, both modified instances are decided to be type 0.

\section{Conclusion}
\label{sec:conclusion}
We proposed an algorithm for compiling a multi-class \acf{BNC} into tractable circuits.
In particular, the class formula of each class is compiled into an OR-decomposable NNF circuit.
These class formulas (circuits) can be used to efficiently compute complete and general reasons for decisions of the \ac{BNC}, which form the basis for computing further explanations. 
Compared to prior work, the proposed algorithm can address multi-class \acp{BNC}, and our experimental results suggest that the algorithm has significant time improvement.

\begin{credits}
\subsubsection{\ackname} 
This work has been partially supported by ARL grant\\
\#W911NF2510095.

\subsubsection{\discintname}
The authors have no competing interests to declare that are relevant to the content of this article.
\end{credits}

\bibliography{references}

\end{document}